\documentclass[twoside,11pt]{article} 
\usepackage{jmlr2e} 
\usepackage{amsmath}
\usepackage{amssymb}
\usepackage[colorlinks]{hyperref}
\usepackage{ifthen}
\usepackage{algpseudocode,algorithm}
\usepackage{graphicx}
\usepackage{natbib}

\newcounter{hypA}
\newenvironment{hypA}{\refstepcounter{hypA}\begin{itemize}
\item[{\bf A\arabic{hypA}}]}{\end{itemize}}
\def\rset{\mathbb{R}}
\def\nset{\mathbb{N}}
\newcommand{\eqsp}{\;}
\newcommand{\eqdef}{\ensuremath{\stackrel{\mathrm{def}}{=}}}
\newcommand{\esp}[1]{\mathbb{E}\left[#1\right]}

\newcommand{\var}[1]{\mathbb{V}\mathrm{ar}\left(#1\right)}
\newcommand{\cov}[2]{\mathbb{C}\mathrm{ov}\left(#1,#2\right)}

\newcommand{\proba}[1]{\mathbb{P}\left\{#1\right\}}
\newcommand{\dlim}{\ensuremath{\stackrel{\mathcal{D}}{\longrightarrow}}}
\newcommand{\plim}{\ensuremath{\stackrel{\mathbb{P}}{\longrightarrow}}}
\newcommand{\normallaw}[2]{\mathcal{N}\left(#1,#2\right)}

\newcommand{\abs}[1]{\left|#1\right|}
\def \one {\bf 1}
\def \e {\mbox{e}}
\def \d {\mbox{d}}

\def \refpop {A}
\def \testpop {B}
\def \numuid {n}
\def \ratioa {\alpha_{\refpop}}
\def \ratiob {\alpha_{\testpop}}
\newcommand{\Xa}[1][]{X^{\refpop}\ifthenelse{\equal{#1}{}}{}{_{#1}}}
\newcommand{\Ya}[1][]{Y^{\refpop}\ifthenelse{\equal{#1}{}}{}{_{#1}}}
\newcommand{\Xb}[1][]{X^{\testpop}\ifthenelse{\equal{#1}{}}{}{_{#1}}}
\newcommand{\Yb}[1][]{Y^{\testpop}\ifthenelse{\equal{#1}{}}{}{_{#1}}}
\newcommand{\epsilona}[1][]{\varepsilon^{\refpop}\ifthenelse{\equal{#1}{}}{}{_{#1}}}
\newcommand{\epsilonb}[1][]{\varepsilon^{\testpop}\ifthenelse{\equal{#1}{}}{}{_{#1}}}
\newcommand{\Xtildea}[1][]{\widetilde{\Xa[#1]}}
\newcommand{\Ytildea}[1][]{\widetilde{\Ya[#1]}}
\newcommand{\Xtildeb}[1][]{\widetilde{\Xb[#1]}}
\newcommand{\Ytildeb}[1][]{\widetilde{\Yb[#1]}}
\newcommand{\empmean}[2]{S^{#1}_{#2}}
\newcommand{\mean}[1]{m_{#1}}
\newcommand{\stddev}[1]{\sigma_{#1}}
\newcommand{\stddevtilde}[1]{\tilde{\stddev{1}}}
\def \correla {\rho_{\refpop}}
\def \correlb {\rho_{\testpop}}
\def \correltildea {\widetilde{\correla}}
\def \correltildeb {\widetilde{\correlb}}
\def \sumXtildea {\empmean{\Xtildea}{\numuid}}
\def \sumYtildea {\empmean{\Ytildea}{\numuid}}
\def \sumXtildeb {\empmean{\Xtildeb}{\numuid}}
\def \sumYtildeb {\empmean{\Ytildeb}{\numuid}}
\def \stddevXtildea {\stddev{\Xtildea}}
\def \stddevYtildea {\stddev{\Ytildea}}
\def \stddevXtildeb {\stddev{\Xtildeb}}
\def \stddevYtildeb {\stddev{\Ytildeb}}
\newcommand{\nonzero}[1][]{\varphi\ifthenelse{\equal{#1}{}}{}{\left(#1\right)}}
\newcommand{\estimator}[1]{\widehat{#1}[\numuid]}

\jmlrheading{}{2015}{}{1/15}{}{Cyrille Dubarry}
\ShortHeadings{Confidence intervals for AB-test}{Cyrille Dubarry}
\firstpageno{1}

\begin{document}

\title{Confidence intervals for AB-test}

\author{\name Cyrille Dubarry \email c.dubarry@criteo.com \\
       \addr Criteo\\
       32, rue Blanche\\
       Paris, France}

\editor{}

\maketitle
\begin{abstract}%
AB-testing is a very popular technique in web companies since it makes it possible to accurately predict the impact of a modification with the simplicity of a random split across users. One of the critical aspects of an AB-test is its duration and it is important to reliably compute confidence intervals associated with the metric of interest to know when to stop the test. In this paper, we define a clean mathematical framework to model the AB-test process. We then propose three algorithms based on bootstrapping and on the central limit theorem to compute reliable confidence intervals which extend to other metrics than the common probabilities of success. They apply to both absolute and relative increments of the most used comparison metrics, including the number of occurrences of a particular event and a click-through rate implying a ratio.
\end{abstract}

\begin{keywords}
AB-test, Confidence interval, Central limit theorem, Ratio of normal variables, Bootstrapping
\end{keywords}

\section{Introduction}
\label{sec:introduction}
Evaluating complex web systems and their impact on user behavior is a challenge of growing importance. Data-driven tools have become very popular in the last decades to help in deciding which algorithm, which website home page, which user interface, etc, provides the best results in terms of some relevant criteria such as the generated revenue, the click-through rate (CTR), the number of visits, or any other business metric. A detailed description of the general data-driven paradigm is available in \citet{darema:2004}.

Different experimention methods are available, \citep[for a primer]{avinash:2006}, and AB-testing, aka split or bucket testing, is wide-spread. For examples and best practices, we refer the reader to \citet{crook:frasca:2009,kohavi:longbotham:2009,kohavi:deng:2012} and references therein. This method compares two versions, $\refpop$ and $\testpop$, of a system by splitting the users randomly into two independent populations to which systems $\refpop$ and $\testpop$ are respectively applied. We use the word \emph{system} in a broad sense here as it can range from being the design of a web page \citep{swanson:2011} to more complex algorithms such as a bidder on a real time bidding ad server \citep{zhang:yuan:wang:2014}. Relevant metrics are then computed on each population and compared to decide which system performs better.

Such comparisons rely on statistical tests to evaluate their significance, see for example \citet{crocker:algina:1986,keppel:1991}, among which Z-tests assess if the neutral hypothesis can be rejected or not at a fixed level of certainty. The simplest example is the one measuring a click-through rate, or any other rate that can only lead to binary values. The click-through rate can be written as the empirical average of Bernoulli random variables equal to $1$ if the user has clicked and to $0$ otherwise. Then, the central limit theorem provides confidence intervals for both the click-through rate in each population and its absolute increment between the two populations \citep[see][for an example]{amazon:abtestmath}. In this case, the asymptotic variance is directly derived from the estimated click-through rate $p$ as $p(1-p)/\numuid$ where $\numuid$ is the number of users.

In practice, a user might click several times. Then the random variables that are averaged are no longer distributed under the Bernoulli law and the asymptotic variance can not be computed in the same way. We show that using such an approximation can even be dangerous through a numerical application to CTR. As stated in \citet{kohavi:longbotham:2009}, we need to use the variance of the number of clicks per user. They also provide confidence intervals for their relative increment using an approximation for the ratio adapted from \citet{willan:briggs:2006} but estimators for the involved variances are not provided for non Bernoulli random variables. Furthermore, these confidence intervals do not take into account the randomness of the number of displays made to users.

The litterature lacks of a formal modeling of the AB-test process. Previous works such as \citet{crook:frasca:2009,kohavi:longbotham:2009,kohavi:deng:2012} mainly focus on applications of this method and do not provide a well-defined statistical framework for the results' analysis. Most available sources for the practitioner are online calculators only dedicated to the Bernoulli case. A primer of the underlying theory applied to AB-test analysis is only given in online references such as \citet{amazon:abtestmath} but they do not go deeply into the statistical modeling and do not cover more general metrics than simple sums of independent Bernoulli random variables. In this paper, we introduce a formal framework for the AB-test process modeling only involving assumptions consistent with the data-driven paradigm. It allows us to prove some statistical properties of the involved estimators, including those based on ratios, and to get numerical methods to approximate the variances involved in the related central limit theorems. We also go beyond that by justifying the use of the bootstrap algorithm \citep{efron:tibshirani:1993} to compute confidence intervals for absolute and relative increments.

The mathematical formalization of the AB-test framework is given in Section \ref{sec:framework}. In Section \ref{sec:mainresults}, we provide exact asymptotic confidence intervals for any kind of metric that is obtained by summing quantities over the users, and for any metric computed as the ratio of such sums. We also get exact asymptotic confidence intervals for both their absolute and relative increments under few assumptions, most of them directly related to the AB-test process. Explicit estimators for the related asymptotic variances are provided. We additionaly show how to use bootstrapping to get confidence intervals when the data cannot be grouped by user, as is commonly the case in the big-data field. Section \ref{sec:experiments} numerically validates our assumptions and the proposed algorithms, while Appendices \ref{appendix:proofs} and \ref{appendix:proofs-clt} give formal proofs of the technical results of Section \ref{sec:mainresults}.

\section{Mathematical Formulation of the AB-test Process}
\label{sec:framework}
In order to translate the AB-test process into a mathematical framework, we introduce some random variables modeling the metrics that one wants to evaluate and the way in which the users are separated into two populations.

More precisely, let $(\Omega, \mathcal{F}, \mathbb{P})$ be a probability space and $\mathbb{E}[\cdot]$ the expectation operator under $\mathbb{P}$. We define a sequence of random vectors on $\rset^4\times \{0,1\}^2$
$$(\Xa[i], \Ya[i], \Xb[i], \Yb[i], \epsilona[i], \epsilonb[i])_{i \geq 1}\eqsp.$$
For each user $i \geq 1$, $\epsilona[i]$ and $\epsilonb[i]$ indicate the population that has been selected for this user:  $\epsilona[i]=1$ (resp. $\epsilonb[i]=1$) if and only if the user $i$ is in population $\refpop$ of size ratio $\ratioa \in [0,1]$ (resp. $\testpop$ of size ratio $\ratiob \in [0,1]$). Note that in general we will have $\ratioa + \ratiob = 1$ but this is not required and our analysis also applies to tests involving more than two populations.
The other variables model metrics of interest for the AB-tester. $\Xa[i]$ and $\Xb[i]$ are the same metric generated by the user $i$ if he was applied to systems $\refpop$ and $\testpop$ respectively. The same stands for $\Ya[i]$ and $\Yb[i]$ which model another metric.

\begin{example}[Comparison of revenue]\label{ex:revenue}
When the AB-tester wants to compare the revenue generated by algorithms $\refpop$ and $\testpop$, he compares the total revenue of each population, normalized by their ratio. They can be written:
$$\frac{1}{\ratioa} \sum_{i|\epsilona[i]=1} \Xa[i] \quad \mbox{and} \quad \frac{1}{\ratiob} \sum_{i|\epsilonb[i]=1} \Xb[i] \eqsp,$$
if $\Xa[i]$ and $\Xb[i]$ are the revenues generated by user $i$ under systems $\refpop$ and $\testpop$ respectively.
Note that, in practice, we can also normalize the total revenues by the real population sizes instead of their ratios and the quantities to compare become:
$$\dfrac{\sum_{i|\epsilona[i]=1} \Xa[i]}{\sum_{i|\epsilona[i]=1}1} \quad \mbox{and} \quad \dfrac{\sum_{i|\epsilona[i]=1} \Xb[i]}{\sum_{i|\epsilona[i]=1} 1} \eqsp.$$
\end{example}

\begin{example}[Comparison of CTR]\label{ex:ctr}
When the AB-tester wants to compare the CTR generated by algorithms $\refpop$ and $\testpop$, he compares the CTR of each population. They can be written:
$$\dfrac{\sum_{i|\epsilona[i]=1} \Xa[i]}{\sum_{i|\epsilona[i]=1} \Ya[i]} \quad \mbox{and} \quad \dfrac{\sum_{i|\epsilona[i]=1} \Xb[i]}{\sum_{i|\epsilona[i]=1} \Yb[i]} \eqsp,$$
if $\Xa[i]$ and $\Xb[i]$ are the clicks generated by user $i$, and $\Ya[i]$ and $\Yb[i]$ the number of displays shown to the same user under systems $\refpop$ and $\testpop$ respectively.
\end{example}

We introduce the following assumptions that will be easily followed in an AB-test setting.
\begin{hypA}
\label{hypA:iid}
The random vectors $(\Xa[i], \Ya[i], \Xb[i], \Yb[i], \epsilona[i], \epsilonb[i])_{i \geq 1},$ are independent and identically distributed.
\end{hypA}

\begin{hypA}
\label{hypA:abtest-independence}
The random vectors $(\Xa[1], \Ya[1], \Xb[1], \Yb[1])$ and $(\epsilona[1], \epsilonb[1])$ are independent.
\end{hypA}

\begin{hypA}
\label{hypA:moment}
The random variables $(\Xa[1], \Ya[1], \Xb[1], \Yb[1])$ are $L_2$-integrable and we define
\begin{equation}\label{eq:mean-def}
\mean{\Xa} \eqdef \esp{\Xa[1]}\eqsp, \eqsp \mean{\Ya} \eqdef \esp{\Ya[1]}\eqsp, \eqsp \mean{\Xb} \eqdef \esp{\Xb[1]}\eqsp, \eqsp \mean{\Yb} \eqdef \esp{\Yb[1]}\eqsp,
\end{equation}
\begin{equation}
\stddev{\Xa}^2 \eqdef \var{\Xa[1]}\eqsp, \eqsp \stddev{\Ya}^2 \eqdef \var{\Ya[1]}\eqsp, \eqsp \stddev{\Xb}^2 \eqdef \var{\Xb[1]}\eqsp, \eqsp \stddev{\Yb}^2 \eqdef \var{\Yb[1]}\eqsp,
\end{equation}
\begin{equation}
\correla \eqdef \dfrac{\cov{\Xa[1]}{\Ya[1]}}{\stddev{\Xa}\stddev{\Ya}} \eqsp, \eqsp \correlb \eqdef \dfrac{\cov{\Xb[1]}{\Yb[1]}}{\stddev{\Xb}\stddev{\Yb}} \eqsp.
\end{equation}
\end{hypA}

\begin{hypA}
\label{hypA:non-negative}
The random variables $(\Xa[1], \Ya[1], \Xb[1], \Yb[1])$ are almost surely non-negative and not almost surely zero, that is
\begin{align*}
& \proba{\Xa[1] < 0} = 0 \eqsp,	&	 \proba{\Xa[1] > 0} > 0 \eqsp,	\\
& \proba{\Ya[1] < 0} = 0 \eqsp,	&	 \proba{\Ya[1] > 0} > 0 \eqsp,	\\
& \proba{\Xb[1] < 0} = 0 \eqsp,	&	 \proba{\Xb[1] > 0} > 0 \eqsp,	\\
& \proba{\Yb[1] < 0} = 0 \eqsp,	&	 \proba{\Yb[1] > 0} > 0 \eqsp.	\\
\end{align*}
\end{hypA}

\begin{hypA}
\label{hypA:epsilon-law}
The random variables $\left(\epsilona[1],\epsilonb[1]\right)$ satisfies:
\begin{enumerate}
\item \label{hypA:epslion-law:marginal} $\epsilona[1]$ and $\epsilonb[1]$ follow Bernoulli laws of respective parameters $\ratioa$ and $\ratiob$.
\item \label{hyp1:epsilon-law:joint} $\epsilona[1]\epsilonb[1] = 0\eqsp.$
\end{enumerate}
\end{hypA}

A user can only be assigned to one population, which is ensured by Assumption A\ref{hypA:epsilon-law}-\ref{hyp1:epsilon-law:joint}. Assumption A\ref{hypA:epsilon-law}-\ref{hypA:epslion-law:marginal} sets the ratio of populations $\refpop$ and $\testpop$ to be respectively $\ratioa$ and $\ratiob$.

Assumption A\ref{hypA:abtest-independence} reflects the fact that the population attribution process does not affect the user reaction to the applied system while Assumption A\ref{hypA:moment} is purely technical. This is the only assumption that is not implied by the AB-test process but it will guarantee the convergence of the estimators. Assumption A\ref{hypA:non-negative} is consistent with the metrics that we are studying. They will typically be zero with a high probability and positive otherwise (for example, the number of clicks).

Finally, Assumption A\ref{hypA:iid} models the un-identifiability of the users. They are all independent and, without prior knowledge, identically distributed. The whole AB-test process relies on this assumption by randomly splitting the users into two populations.

It is worthwhile to note that the metrics of interest $(\Xa[i], \Ya[i], \Xb[i], \Yb[i])_{i \geq 1}$ are defined for each user and for each system, independently of the population split. The AB-test process will give access to only $\Xa[i]$ or $\Xb[i]$ for a given user $i$, but they can still both be defined even when they are not observed. This is the main interest of this modeling that allows us to write those variables independently of the population. Furthermore, we circumvent the issue of having hidden variables by introducing a new set of variables that will always be observed. To that purpose, we simply set $\Xa[i]$ to $0$ when it is not observed, i.e. when the user $i$ is not in population $\refpop$. This is formalized in the following definition.
\begin{definition}\label{def:tilde}
For each user $i \geq 1$, we define
\begin{equation*}
\Xtildea[i] \eqdef \dfrac{\epsilona[i]\Xa[i]}{\ratioa} \eqsp, \eqsp \Ytildea[i] \eqdef \dfrac{\epsilona[i]\Ya[i]}{\ratioa} \eqsp, \eqsp \Xtildeb[i] \eqdef \dfrac{\epsilonb[i]\Xb[i]}{\ratiob} \eqsp, \eqsp \Ytildeb[i] \eqdef \dfrac{\epsilonb[i]\Yb[i]}{\ratiob}\eqsp.
\end{equation*}
\end{definition}

\begin{remark} \label{rem:iid}
We trivially obtain from Assumption A\ref{hypA:iid} that the random vectors $(\Xtildea[i], \Ytildea[i], \Xtildeb[i], \Ytildeb[i])_{i \geq 1}$ are independent and identically distributed.
\end{remark}

Using Definition \ref{def:tilde}, sums of the form
$$\frac{1}{\ratioa} \sum_{i|\epsilona[i]=1} \Xa[i]\eqsp,$$
can by re-written in a more appealing way as
$$\sum_{i=1}^{\numuid} \Xtildea[i]\eqsp,$$
where the random variables $(\Xtildea[i])_{i \geq 1}$ are summed on \emph{all} the users independently of their population, which leads to the following sum definitions for any number of users $\numuid \in \nset$:
\begin{equation}\label{eq:sum-def}
\sumXtildea \eqdef \frac{1}{\numuid}\sum_{i=1}^{\numuid} \Xtildea[i]\eqsp,\eqsp \sumYtildea \eqdef \frac{1}{\numuid}\sum_{i=1}^{\numuid} \Ytildea[i]\eqsp,\eqsp \sumXtildeb \eqdef \frac{1}{\numuid}\sum_{i=1}^{\numuid} \Xtildeb[i]\eqsp,\eqsp \sumYtildeb \eqdef \frac{1}{\numuid}\sum_{i=1}^{\numuid} \Ytildeb[i]\eqsp.
\end{equation}
In the case of Example \ref{ex:revenue}, we will have to compare either two sums over the same indices $\sumXtildea$ and $\sumXtildeb$ when the normalization is done by the population ratios ; or two ratios of sums over the same indices $\sumXtildea/\sumYtildea$ and $\sumXtildeb/\sumYtildeb$, where $\Ya \equiv 1$ and $\Yb \equiv 1$, when the normalization is done by the real population sizes. In the case of Example \ref{ex:ctr}, the ratios to compare become similarly $\sumXtildea/\sumYtildea$ and $\sumXtildeb/\sumYtildeb$.

Writting the estimators this way validates the use of the bootstrap technique \citep{efron:tibshirani:1993} to get confidence intervals. For the relative increments of the metrics of interest, this can be done through the study of ratio:
\begin{equation}\label{eq:sumtilderatio}
\dfrac{\sumXtildeb}{\sumXtildea} \quad \mbox{and} \quad \dfrac{\sumXtildeb/\sumYtildeb}{\sumXtildea/\sumYtildea}\eqsp.
\end{equation}
Three algorithms will be derived in the following Section to get confidence intervals on such quantities.

\section{Estimator Convergence and Algorithms for Confidence Intervals}
\label{sec:mainresults}

The previous modeling has been designed to translate AB-test metrics into functions of sums of i.i.d. variables as in \eqref{eq:sumtilderatio}.
The i.i.d. property allows us to design and validate a bootstrap technique to get confidence intervals, and dealing only with sums adds the ability to derive central limit theorems for all the metrics and their increments (both absolute and relative).

\subsection{Confidence Interval Computation}
\label{subsec:bootstrap}

According to Remark \ref{rem:iid}, the random vectors $(\Xtildea[i], \Ytildea[i], \Xtildea[i], \Ytildeb[i])_{i \geq 1}$ are i.i.d., and by Definition \ref{def:tilde} we have for $i \geq 1$
\begin{equation*}
\abs{\Xtildea[i]} \leq \frac{1}{\ratioa} \abs{\Xa[i]} \eqsp, \eqsp \abs{\Ytildea[i]} \leq \frac{1}{\ratioa} \abs{\Ya[i]} \eqsp, \eqsp \abs{\Xtildeb[i]} \leq \frac{1}{\ratioa} \abs{\Xb[i]} \eqsp, \eqsp \abs{\Ytildeb[i]} \leq \frac{1}{\ratioa} \abs{\Yb[i]}\eqsp.
\end{equation*}
Assumption A\ref{hypA:moment} then shows that $(\Xtildea[i], \Ytildea[i], \Xtildea[i], \Ytildeb[i])_{i \geq 1}$ are $L_1$ integrable. We thus can apply the law of large numbers to the sums of interest $(\sumXtildea, \sumYtildea, \sumXtildeb, \sumYtildeb)$ and show that they converge to $(\mean{\Xa}, \mean{\Ya},\mean{\Xb},\mean{\Yb})$. We then get that for any continuous function $f$, the quantity $f(\sumXtildea, \sumXtildea, \sumYtildea, \sumYtildeb)$ is a consistent estimator of $f(\mean{\Xa}, \mean{\Ya},\mean{\Xb},\mean{\Yb})$.

The case of a ratio is dealt with by introducting the following transformation.
\begin{definition}\label{def:nonzero}
We define the function $\nonzero$ from $\rset$ to $\rset^*$ defined by
\begin{equation*}
\forall x \in \rset \eqsp, \quad \nonzero[x] \eqdef \left\{ \begin{array}{l@{\quad,\mbox{if }}l} 1 &x = 0\eqsp,\\ x &  x \neq 0 \eqsp. \end{array}\right.
\end{equation*}
\end{definition}
We will apply $\nonzero$ to all the denominators in the following theorems, and, according to the positiveness ensured by Assumption A\ref{hypA:non-negative}, the ratios are continuous functions of the non-zero sums. It is only a technical point, as in practice we would not define the ratio for a null denominator. In theoretical applications, Lemma \ref{lem:nonzero-probaconv} in Appendix \ref{appendix:proofs} allows us to replace the sums by their non-zero versions obtained by applying the operator $\nonzero$, but for the sake of simplicity we will not use it when describing the bootstrap.

If we denote by $D$ the distribution of $(\Xtildea[1], \Ytildea[1], \Xtildeb[1], \Ytildeb[1])$, then all the quantities that we are estimating can be written as a functional $F(D) \eqdef f(\mean{\Xa},\mean{\Ya},\mean{\Xb},\mean{\Yb})$, and their estimators are asymptotically normal as shown in the relevant Propositions of Section \ref{subsec:clt}. The link between estimators, $f$, $F$, and their central limit theorem result is summarized in Table \ref{tab:link-estimators}.
\begin{table}
\begin{center}
\begin{tabular}{||c|c|c|c||}
\hline\hline
Estimator										&	$f(x,y,x',y')$		&	$F(D)$			& CLT \\
\hline \hline
$\sumXtildeb - \sumXtildea$						&	$x'-x$			&	$\mean{\Xb} - \mean{\Xa}$	&	Prop. \ref{prop:clt-diff} \\
\hline
$\sumXtildeb / \sumXtildea$						&	$x'/x$			&	$\mean{\Xb} / \mean{\Xa}$ 	&	Prop. \ref{prop:clt-relative-diff}\\
\hline
$\sumXtildeb / \sumYtildeb - \sumXtildea / \sumYtildea$	&	$x'/y' - x/y	$		&	$\mean{\Xb} / \mean{\Yb} -  \mean{\Xa} / \mean{\Ya}$	&	Prop. \ref{prop:clt-ratio-absolute-diff} \\
\hline
$\dfrac{\sumXtildeb / \sumYtildeb}{\sumXtildea / \sumYtildea}$	&	$\dfrac{x'/y'}{x/y}$		&	$\dfrac{\mean{\Xb} / \mean{\Yb}}{ \mean{\Xa} / \mean{\Ya}}$	&	Prop. \ref{prop:clt-ratio-relative-diff} \\
\hline \hline
\end{tabular}
\end{center}
\caption{Different estimators of interest}\label{tab:link-estimators}
\end{table}

\paragraph{Bootstrapping}
In this specific framework, bootstrapping can be used by randomly selecting $\numuid$ users (possibly picking the same user several times) and computing the estimator with this random set of users. Repeating this $M$ times provides an empirical distribution of the estimator of $F(D)$. The $M$ estimator values can be computed with only one pass on the dataset using an online version of bootstrapping described in \citet{oza:russel:2001,oza:2005}.

For each user $i$, a Poisson random variable $Z_i$ is simulated and the current user is included $Z_i$ times.
\begin{algorithm}[tb]
    \begin{algorithmic}[1]
        \caption{\quad Online bootstrapping}\label{alg:online-bootstrapping}
        \State \underline{\em Inputs:} A dataset $(I_l, \epsilona[I_l]x_l^{\refpop}, \epsilona[I_l]y_l^{\refpop}, \epsilonb[I_l]x_l^{\testpop}, \epsilonb[I_l]y_l^{\testpop})_{l = 1}^M$ and random variables $(Z_i^m)_{m=1:M}^{i=1:\numuid}$.
	\State \underline{\em Initialization:} Set $\left(\Gamma_{k,m}\right)_{1 \leq k \leq 4}^{1 \leq m \leq M}$ a null $4\times M$ matrix of sum estimators.
        \State \underline{\em Loop on the data set:}
        \For{$l$ from $1$ to $L$}
         	\State Set $i = I_l$.
	 	\For{$m$ from $1$ to $M$}
			\State Set $\Gamma_{1,m} = \Gamma_{1,m} + Z_i^m \epsilona[i] x_l^{\refpop} / \ratioa$.
			\State Set $\Gamma_{2,m} = \Gamma_{2,m} + Z_i^m \epsilona[i] y_l^{\refpop} / \ratioa$.
			\State Set $\Gamma_{3,m} = \Gamma_{3,m} + Z_i^m \epsilonb[i] x_l^{\testpop} / \ratiob$.
			\State Set $\Gamma_{4,m} = \Gamma_{4,m} + Z_i^m \epsilonb[i] y_l^{\testpop} / \ratiob$.
		\EndFor
        \EndFor
	\State \underline{\em Computation of the estimators}
	\For{$m$ from $1$ to $M$}
		\State Set $\numuid_m \eqdef \sum_{i=1}^{\numuid} Z_i^m$.
		\State Set $\widehat{F_m} \eqdef f\left(\frac{1}{\numuid_m}\Gamma_{1,m}, \frac{1}{\numuid_m}\Gamma_{2,m}, \frac{1}{\numuid_m}\Gamma_{3,m}, \frac{1}{\numuid_m}\Gamma_{4,m}\right)$.
	\EndFor
        \State \underline{\em Outputs:} $\left(\widehat{F_m}\right)_{m=1}^M$.
    \end{algorithmic}
\end{algorithm}
The full procedure is detailed in Algorithm \ref{alg:online-bootstrapping} and works well even if the dataset is not grouped by user. In this case, each line $l$ of the dataset is associated to a user $i = I_l$ and contains a vector $(\epsilona[I_l]x_l^{\refpop}, \epsilona[I_l]y_l^{\refpop}, \epsilonb[I_l]x_l^{\testpop}, \epsilonb[I_l]y_l^{\testpop})$ such that for any $i \geq 1$
\begin{align*}
&\Xa[i] = \sum_{l=1|I_l=i}^L x_l^{\refpop} \eqsp, & \Ya[i] = \sum_{l=1|I_l=i}^L y_l^{\refpop} \eqsp, \\
&\Xb[i] = \sum_{l=1|I_l=i}^L x_l^{\testpop} \eqsp, & \Yb[i] = \sum_{l=1|I_l=i}^L y_l^{\testpop} \eqsp.
\end{align*}
It relies on a pseudo-random generator that is able to generate $M$ Poisson variables $(Z_i^m)_{1 \leq m \leq M}$ for each user $i$.

\paragraph{Confidence interval algorithms}
The $M$ estimators $\left(\widehat{F_m}\right)_{m=1}^M$ obtained in Algorithm \ref{alg:online-bootstrapping} can then be used to derive empirical quantiles and obtain confidence intervals with Algorithm \ref{alg:ci-bootstrap}.
\begin{algorithm}[tb]
    \begin{algorithmic}[1]
        \caption{\quad Confidence interval with bootstrapping}\label{alg:ci-bootstrap}
        \State \underline{\em Inputs:} The bootstrap distribution $\left(\widehat{F_m}\right)_{m=1}^M$ given by Algorithm \ref{alg:online-bootstrapping} and a confidence level $q$.
	\State Compute $F_{\mbox{min}}$ as the empirical quantile of $\left(\widehat{F_m}\right)_{m=1}^M$ of order $(1-q)/2$.
	\State Compute $F_{\mbox{max}}$ as the empirical quantile of $\left(\widehat{F_m}\right)_{m=1}^M$ of order $(1+q)/2$.
        \State \underline{\em Outputs:} $\displaystyle \left[ F_{\mbox{min}}, F_{\mbox{max}} \right]$.
    \end{algorithmic}
\end{algorithm}
However, quantile approximation for accurate confidence intervals requires $M$ to be big enough and Algorithm \ref{alg:ci-bootstrap} is only feasible if the number of users $\numuid$ is small enough.

Another way of computing confidence intervals is to use one of the central limit theorems stated in Section \ref{subsec:clt} on the condition that the implied variances can be easily estimated from the data. The resulting algorithm is given in Algorithm \ref{alg:clt-ci} where we use the normal cumulative density function $N$ defined by
\begin{equation}
\forall x \in \rset \eqsp, \quad N(x) \eqdef \int_{-\infty}^x \dfrac{\e^{-t^2/2}}{\sqrt{2\pi}} \d t \eqsp.
\end{equation}
\begin{algorithm}[tb]
    \begin{algorithmic}[1]
        \caption{\quad Confidence interval with CLT}\label{alg:clt-ci}
        \State \underline{\em Inputs:} $(\Xtildea[i], \Ytildea[i], \Xtildeb[i], \Ytildeb[i])_{i=1}^{\numuid}$ and a confidence level $q$.
	\State Set $\displaystyle s \eqdef N^{-1}\left( \dfrac{1+q}{2} \right)$.
	\State Estimate the asymptotic variance $\widehat{\sigma_{\numuid}}^2$ using the relevant Proposition (see Table \ref{tab:link-estimators}).
        \State \underline{\em Outputs:} $\displaystyle \left[f\left(\sumXtildea,\sumYtildea,\sumXtildeb,\sumYtildeb\right) - s \widehat{\sigma_{\numuid}}, f\left(\sumXtildea,\sumYtildea,\sumXtildeb,\sumYtildeb\right) + s \widehat{\sigma_{\numuid}}\right]$.
    \end{algorithmic}
\end{algorithm}

In practice, the data is not aggregated by user and we have to do so as a first step in order to get the vectors $(\Xtildea[i], \Ytildea[i], \Xtildeb[i], \Ytildeb[i])_{i \geq 1}$ and estimate the related variances and covariances. This can be quite costly as it requires more than one reading of the dataset if the user can be found in several lines. In the case where each user appears only once, this will be the quicker algorithm as it does not need any simulation.

We can take advantage of both Algorithms \ref{alg:ci-bootstrap} and \ref{alg:clt-ci} by using bootstrapping to approximate the estimator variance and the asymptotic normality to derive confidence intervals as described in Algorithm \ref{alg:fast-ci}. The variance estimation only requires a few number of bootstraps $M$ and the dataset is read only once. This algorithm will be shown in Section \ref{sec:experiments} to perform better than Algorithm \ref{alg:ci-bootstrap} for a given computational cost. Though, this algorithm relies on an asymptotic regime and is relevant only when the number of users $\numuid$ is large enough. Otherwise, pure bootstrapping may be a better alternative as it works for any value of $\numuid$.

\begin{algorithm}[tb]
    \begin{algorithmic}[1]
        \caption{\quad Confidence interval with bootstrapping and CLT}\label{alg:fast-ci}
        \State \underline{\em Inputs:} The bootstrap distribution $\left(\widehat{F_m}\right)_{m=1}^M$ given by Algorithm \ref{alg:online-bootstrapping} and a confidence level $q$.
	\State Set $\displaystyle s \eqdef N^{-1}\left( \dfrac{1+q}{2} \right)$.
	\State Set $\displaystyle \left(\widehat{\sigma_{\numuid}^F}\right)^2 \eqdef \dfrac{1}{M-1} \sum_{m=1}^M \left(\widehat{F_m} - \frac{1}{M}\sum_{p=1}^M\widehat{F_p}\right)^2 $
        \State \underline{\em Outputs:} $\displaystyle \left[f\left(\sumXtildea,\sumYtildea,\sumXtildeb,\sumYtildeb\right) - s \widehat{\sigma_{\numuid}^F}, f\left(\sumXtildea,\sumYtildea,\sumXtildeb,\sumYtildeb\right) + s \widehat{\sigma_{\numuid}^F}\right]$.
    \end{algorithmic}
\end{algorithm}

\subsection{Central Limit Theorem}
\label{subsec:clt}
We now check that the estimators given in Table \ref{tab:link-estimators} all satisfy a central limit theorem. For improved readability, proofs have been postponed to Appendix \ref{appendix:proofs-clt}
\begin{theorem}[Central limit theorem] \label{th:clt}
Under Assumptions A\ref{hypA:iid}-\ref{hypA:epsilon-law}, the vector $(\sumXtildea, \sumYtildea, \sumXtildeb, \sumYtildeb)$, defined in \eqref{eq:sum-def}, satisfies the following central limit theorem
\begin{equation*}
\sqrt{\numuid}
\left(
	\begin{array}{c}
		\sumXtildea - \mean{\Xa} \\
		\sumYtildea - \mean{\Ya} \\
		\sumXtildeb - \mean{\Xb} \\
		\sumYtildeb - \mean{\Yb}
	\end{array}
\right)
\dlim
\normallaw
	{\left(
		\begin{array}{c}
			0 \\
			0 \\
			0 \\
			0
		\end{array}
	\right)}
	{\left(
		\Sigma\left(\Xtildea[1], \Ytildea[1], \Xtildeb[1], \Ytildeb[1] \right)
	\right)}\eqsp,
\end{equation*}
where $\Sigma\left(\Xtildea[1], \Ytildea[1], \Xtildeb[1], \Ytildeb[1] \right)$ is the covariance matrix of $\left(\Xtildea[1], \Ytildea[1], \Xtildeb[1], \Ytildeb[1] \right)$ defined by the variances
\begin{align}
\label{eq:varxtildea-def}
  \stddevXtildea^2 \eqdef \var{\Xtildea[1]} &= \dfrac{1}{\ratioa}\stddev{\Xa}^2 + \dfrac{1-\ratioa}{\ratioa} \mean{\Xa}^2 \eqsp,\\
\label{eq:varytildea-def}
  \stddevYtildea^2\eqdef\var{\Ytildea[1]} &= \dfrac{1}{\ratioa}\stddev{\Ya}^2 + \dfrac{1-\ratioa}{\ratioa} \mean{\Ya}^2  \eqsp,\\
\label{eq:varxtildeb-def}
 \stddevXtildeb^2 \eqdef \var{\Xtildeb[1]}& = \dfrac{1}{\ratiob}\stddev{\Xb}^2 + \dfrac{1-\ratiob}{\ratiob} \mean{\Xb}^2  \eqsp,\\
\label{eq:varytildeb-def}
 \stddevYtildeb^2\eqdef \var{\Ytildeb[1]}& = \dfrac{1}{\ratiob}\stddev{\Yb}^2 + \dfrac{1-\ratiob}{\ratiob} \mean{\Yb}^2   \eqsp,
\end{align}
the covariances inside each population
\begin{align}
\label{eq:rhotildea-def}
 \cov{\Xtildea[1]}{\Ytildea[1]} = \dfrac{1}{\ratioa}\correla\stddev{\Xa}\stddev{\Ya} + \dfrac{1-\ratioa}{\ratioa} \mean{\Xa}\mean{\Ya}& \eqdef \correltildea \stddevXtildea \stddevYtildea \eqsp, \\
\label{eq:rhotildeb-def}
 \cov{\Xtildeb[1]}{\Ytildeb[1]} = \dfrac{1}{\ratiob}\correlb\stddev{\Xb}\stddev{\Yb} + \dfrac{1-\ratiob}{\ratiob} \mean{\Xb}\mean{\Yb}& \eqdef \correltildeb \stddevXtildeb \stddevYtildeb \eqsp,
\end{align}
and the cross population covariances
\begin{align*}
& \cov{\Xtildea[1]}{\Xtildeb[1]} = -\mean{\Xa}\mean{\Xb}\eqsp,\\
& \cov{\Xtildea[1]}{\Ytildeb[1]} = -\mean{\Xa}\mean{\Yb}\eqsp,\\
& \cov{\Ytildea[1]}{\Xtildeb[1]} = -\mean{\Ya}\mean{\Xb}\eqsp,\\
& \cov{\Ytildea[1]}{\Ytildeb[1]} = -\mean{\Ya}\mean{\Yb}\eqsp.\\
\end{align*}
\end{theorem}

The convergence is done at rate $\sqrt{\numuid}$ where $\numuid$ is the total number of users, and not the number of users in a population. However, the variance of each estimator decreases with its relative population size thanks to factors $\ratioa$ and $\ratiob$ found in the denominators of the four variances.

Furthermore, these variances are composed of two terms. One that comes purely from the variance of the metrics of interest (ex: $\stddev{\Xa}^2$) and another one added by the AB-test process which randomly attributes each user to a population (ex: $(1-\ratioa)\mean{\Xa}^2$). They can be understood when looking at extreme cases. When population $\refpop$ includes all the users, i.e. $\ratioa = 1$, the randomness of the AB-test process disappears and we simply get $\var{\Xtildea[1]} = \stddev{\Xa}^2$. On the other hand, if the metric of interest $\Xa$ is purely deterministic, let's say $\Xa \equiv 1$ in which case we are interested in the number of users per population, then the variance becomes $\frac{1-\ratioa}{\ratioa}$ which is the variance of $\epsilona[1]/\ratioa$. However, in practice, we often have $\stddev{\Xa} >> \mean{\Xa}$ and the second term becomes almost negligible.

Another fact shown by Theorem \ref{th:clt} is that the metrics of the two populations are not independent! This was actually intuitive as when a user is associated to one population and thus included in the corresponding sum, the other population looses this user. If $\mean{\Xa}$ and $\mean{\Xb}$ are positive, then the correlation is negative following the previous intuition.

Finally, Theorem \ref{th:clt} provides the asymptotic distribution of the joint law of the four empirical averages we are interested in to compare the two populations. Simple linear combinations such as $\sumXtildeb-\sumXtildea$ remain asymptoticaly normal and confidence intervals can easily been derived as stated in Proposition \ref{prop:clt-diff}. This allows for comparing, for example, the absolute increment of the number of displays per user generated by the two algorithms $\refpop$ and $\testpop$.

\begin{proposition}[CLT for $f(x,y,x',y') = x' - x$]\label{prop:clt-diff}
Under Assumptions A\ref{hypA:iid}-\ref{hypA:epsilon-law}, the absolute increment $\sumXtildeb- \sumXtildea$ satisfies the following central limit theorem
\begin{equation*}
\sqrt{\numuid} \left[\left(\sumXtildeb- \sumXtildea\right) - \left(\mean{\Xb}-\mean{\Xa}\right)\right]
 \dlim \normallaw{0}{\stddevXtildea^2 + \stddevXtildeb^2 + 2\mean{\Xa}\mean{\Xb}} \eqsp,
\end{equation*}
where $\stddevXtildea$ and $\stddevXtildeb$ are defined respectively in \eqref{eq:varxtildea-def} and \eqref{eq:varxtildeb-def}.
\end{proposition}

When coming to confidence intervals for relative increments such as $\sumXtildeb/\sumXtildea$, or for ratio metrics such as $\sumXtildea/\sumYtildea$, without further steps, one would need to compute quantiles of the ratio of two correlated normal random variables. This problem is known to be difficult and has been discussed for decades, see \citet{marsaglia:2006} and references therein.

However, such ratios can themselves be shown to be asymptotically normal in our setup as stated in Propositions \ref{prop:clt-relative-diff} and \ref{prop:clt-ratio}.

\begin{proposition}[CLT for $f(x,y,x',y') = x'/x$]\label{prop:clt-relative-diff}
Under Assumptions A\ref{hypA:iid}-\ref{hypA:epsilon-law}, the ratio $\sumXtildeb/\nonzero[\sumXtildea]$ satisfies the following central limit theorem
\begin{equation*}
\sqrt{\numuid}\left( \dfrac{\sumXtildeb}{\nonzero[\sumXtildea]} - \dfrac{\mean{\Xb}}{\mean{\Xa}}\right)
\dlim \normallaw{0}{\left(\dfrac{\mean{\Xb}}{\mean{\Xa}}\right)^2 \left[ \left(\dfrac{\stddevXtildea}{\mean{\Xa}}\right)^2 + \left(\dfrac{\stddevXtildeb}{\mean{\Xb}}\right)^2 + 2 \right]} \eqsp,
\end{equation*}
where $\stddevXtildea$ and $\stddevXtildeb$ are defined respectively in \eqref{eq:varxtildea-def} and \eqref{eq:varxtildeb-def} and $\nonzero$ in Definition \ref{def:nonzero}.
\end{proposition}

Following similar steps, we can now derive central limit theorems for ratio of the form $\sumXtildea/\sumYtildea$ which allows us to get confidence intervals for metrics such as CTR as in Example \ref{ex:ctr}.

\begin{proposition}\label{prop:clt-ratio}
Under Assumptions A\ref{hypA:iid}-\ref{hypA:epsilon-law}, the ratio $\sumXtildea/\nonzero[\sumYtildea]$ and $\sumXtildeb/\nonzero[\sumYtildeb]$ satisfy the following central limit theorem
\begin{equation*}
\sqrt{\numuid}\left(\begin{array}{c}
\dfrac{\sumXtildea}{\nonzero[\sumYtildea]} - \dfrac{\mean{\Xa}}{\mean{\Ya}} \\
\dfrac{\sumXtildeb}{\nonzero[\sumYtildeb]} - \dfrac{\mean{\Xb}}{\mean{\Yb}}
\end{array}\right)
 \dlim \normallaw{0}{\left(\begin{array}{cc} V_{\refpop} & 0 \\ 0 & V_{\testpop} \end{array}\right)}\eqsp,
\end{equation*}
with
\begin{align}
\label{eq:var-ratio-a}
&V_{\refpop} \eqdef \left(\dfrac{\mean{\Xa}}{\mean{\Ya}}\right)^2 \left[ \left(\dfrac{\stddevXtildea}{\mean{\Xa}}\right)^2 + \left(\dfrac{\stddevYtildea}{\mean{\Ya}}\right)^2 - 2 \correltildea \dfrac{\stddevXtildea}{\mean{\Xa}}\dfrac{\stddevYtildea}{\mean{\Ya}} \right] \eqsp,\\
\label{eq:var-ratio-b}
&V_{\testpop} \eqdef \left(\dfrac{\mean{\Xb}}{\mean{\Yb}}\right)^2 \left[ \left(\dfrac{\stddevXtildeb}{\mean{\Xb}}\right)^2 + \left(\dfrac{\stddevYtildeb}{\mean{\Yb}}\right)^2 - 2 \correltildeb \dfrac{\stddevXtildeb}{\mean{\Xb}}\dfrac{\stddevYtildeb}{\mean{\Yb}} \right] \eqsp,
\end{align}
where $\stddevXtildea$, $\stddevYtildea$, $\stddevXtildeb$, $\stddevYtildeb$, $\correltildea$ and $\correltildeb$ are respectively defined in \eqref{eq:varxtildea-def}, \eqref{eq:varytildea-def}, \eqref{eq:varxtildeb-def}, \eqref{eq:varytildeb-def},  \eqref{eq:rhotildea-def} and \eqref{eq:rhotildeb-def}.
\end{proposition}

One can remark that whereas $\sumXtildea$ and $\sumXtildeb$ are asymptotically correlated, as well as $\sumYtildea$ and $\sumYtildeb$, the ratio $\sumXtildea/\nonzero[\sumYtildea]$ and $\sumXtildeb/\nonzero[\sumYtildeb]$ are not. This can be explained by recalling that the correlation of the non-ratio metrics is due to the fact that adding a user to one sum, excludes him from the other one, resulting in a negative correlation. On the contrary, ratios inside each population are independent of the scale of the individual sums, and their correlation vanishes asymptotically.

We can now derive central limit theorems for both the absolute and relative differences of ratios. This is done in Propositions \ref{prop:clt-ratio-absolute-diff} and \ref{prop:clt-ratio-relative-diff} respectively.

\begin{proposition}[CLT for $f(x,y,x',y') = x'/y' - x/y$]
\label{prop:clt-ratio-absolute-diff}
Under Assumptions A\ref{hypA:iid}-\ref{hypA:epsilon-law}, the ratios $\sumXtildea/\nonzero[\sumYtildea]$ and $\sumXtildeb/\nonzero[\sumYtildeb]$ satisfy the following central limit theorem
\begin{equation*}
\sqrt{\numuid}\left[
\left(\dfrac{\sumXtildeb}{\nonzero[\sumYtildeb]} - \dfrac{\sumXtildea}{\nonzero[\sumYtildea]}\right) - \left(\dfrac{\mean{\Xb}}{\mean{\Yb}} - \dfrac{\mean{\Xa}}{\mean{\Ya}}\right)
\right]
 \dlim \normallaw{0}{V_{\refpop}+V_{\testpop}}\eqsp,
\end{equation*}
where $V_{\refpop}$ and $V_{\testpop}$ are defined respectively in \eqref{eq:var-ratio-a} and \eqref{eq:var-ratio-b}.
\end{proposition}

\begin{proposition}[CLT for $f(x,y,x',y') = \frac{x'/y'}{x/y}$]
\label{prop:clt-ratio-relative-diff}
Under Assumptions A\ref{hypA:iid}-\ref{hypA:epsilon-law}, we have the following central limit theorem
\begin{multline*}
\sqrt{\numuid}\left(
\dfrac{\sumXtildeb/\nonzero[\sumYtildeb]}{\nonzero[\sumXtildea]/\nonzero[\sumYtildea]} - \dfrac{\mean{\Xb}/\mean{\Yb}}{\mean{\Xa}/\mean{\Ya}}
\right)\\
 \dlim \normallaw{0}{\left(\dfrac{\mean{\Xb}/\mean{\Yb}}{\mean{\Xa}/\mean{\Ya}}\right)^2 \left[ \left(\dfrac{\sqrt{V_{\refpop}}}{\mean{\Xa}/\mean{\Ya}}\right)^2 + \left(\dfrac{\sqrt{V_{\testpop}}}{\mean{\Xb}/\mean{\Yb}}\right)^2 \right]}\eqsp,
\end{multline*}
where $V_{\refpop}$ and $V_{\testpop}$ are defined in Proposition \ref{prop:clt-ratio}.
\end{proposition}

\subsection{Variance Estimation}
Algorithm \ref{alg:clt-ci} defined in Section \ref{subsec:bootstrap} relies on estimators of the asymptotic variance given in Propositions \ref{prop:clt-diff}, \ref{prop:clt-relative-diff}, \ref{prop:clt-ratio-absolute-diff}, and \ref{prop:clt-ratio-relative-diff}. All the related variances are given as a continuous function of $\mean{\Xa}$, $\mean{\Ya}$, $\mean{\Xb}$, $\mean{\Yb}$ defined in \eqref{eq:mean-def}, and  $\stddevXtildea$, $\stddevYtildea$, $\stddevXtildeb$, $\stddevYtildeb$, $\correltildea$, $\correltildeb$ respectively defined in \eqref{eq:varxtildea-def}, \eqref{eq:varytildea-def}, \eqref{eq:varxtildeb-def}, \eqref{eq:varytildeb-def},  \eqref{eq:rhotildea-def} and \eqref{eq:rhotildeb-def}. According to the continuous mapping theorem, we thus only need to get consistent estimators $\estimator{\mean{\Xa}}$, $\estimator{\mean{\Ya}}$, $\estimator{\mean{\Xb}}$, $\estimator{\mean{\Yb}}$, $\estimator{\stddevXtildea}$, $\estimator{\stddevYtildea}$, $\estimator{\stddevXtildeb}$, $\estimator{\stddevYtildeb}$, $\estimator{\correltildea}$, $\estimator{\correltildeb}$ of these ten quantities to derive consistent estimators of the asymptotic variances stated in Propositions \ref{prop:clt-diff}, \ref{prop:clt-relative-diff}, \ref{prop:clt-ratio-absolute-diff}, and \ref{prop:clt-ratio-relative-diff}.

The mean estimators are easily obtained from $\sumXtildea$, $\sumYtildea$, $\sumXtildeb$, and $\sumYtildeb$:
\begin{equation*}
\estimator{\mean{\Xa}} \eqdef \sumXtildea \eqsp, \quad
\estimator{\mean{\Ya}} \eqdef \sumYtildea \eqsp, \quad
\estimator{\mean{\Xb}} \eqdef \sumXtildeb \eqsp, \quad
\estimator{\mean{\Yb}} \eqdef \sumYtildeb \eqsp.
\end{equation*}

The variance estimators can be computed directly from the random variables $(\Xtildea[i], \Ytildea[i], \Xtildeb[i], \Ytildeb[i])_{i \geq 1}$ without estimating in a first step  $\stddev{\Xa}$, $\stddev{\Ya}$, $\stddev{\Xb}$, $\stddev{\Yb}$:
\begin{align*}
&\estimator{\stddevXtildea}^2 \eqdef \dfrac{1}{\numuid-1} \sum_{i=1}^{\numuid} \left(\Xtildea[i]-\estimator{\mean{\Xa}}\right)^2\eqsp, &\estimator{\stddevYtildea}^2 \eqdef \dfrac{1}{\numuid-1} \sum_{i=1}^{\numuid} \left(\Ytildea[i]-\estimator{\mean{\Ya}}\right)^2\eqsp,\\
&\estimator{\stddevXtildeb}^2 \eqdef \dfrac{1}{\numuid-1} \sum_{i=1}^{\numuid} \left(\Xtildeb[i]-\estimator{\mean{\Xb}}\right)^2\eqsp, &\estimator{\stddevYtildeb}^2 \eqdef \dfrac{1}{\numuid-1} \sum_{i=1}^{\numuid} \left(\Ytildeb[i]-\estimator{\mean{\Yb}}\right)^2\eqsp.
\end{align*}

Finally, the correlation estimators are obtained in a similar way:
\begin{align*}
&\estimator{\correltildea} \eqdef \dfrac{1}{\estimator{\stddevXtildea}\estimator{\stddevYtildea}} \times \dfrac{1}{\numuid-1} \sum_{i=1}^{\numuid} \left(\Xtildea[i]-\estimator{\mean{\Xa}}\right)\left(\Ytildea[i]-\estimator{\mean{\Ya}}\right) \eqsp,\\
&\estimator{\correltildeb} \eqdef \dfrac{1}{\estimator{\stddevXtildeb}\estimator{\stddevYtildeb}} \times \dfrac{1}{\numuid-1} \sum_{i=1}^{\numuid} \left(\Xtildeb[i]-\estimator{\mean{\Xb}}\right)\left(\Ytildeb[i]-\estimator{\mean{\Yb}}\right) \eqsp.
\end{align*}

\section{Numerical Application to CTR Confidence Intervals}
\label{sec:experiments}
We use a real dataset described in Section \ref{subsec:dataset} to numerically demonstrate the proposed algorithms. Blank AB-tests are simulated over this dataset to validate the user independent assumption in Section \ref{subsec:assumption-validation} and to compare the bootstrap algorithms in Section \ref{subsec:algo-comparison}. Blank AB-tests are of particular interest here since we know that whichever the metric of interest, its increment should be $0$. This allows to easily check that a given confidence interval contains the true value it aims to estimate.

\subsection{Dataset Description}
\label{subsec:dataset}
The dataset used in this paper is publically accessible from the KDD Cup website \citet{kdd:2012}.
It has been built out of search session log messages containing one line per search. Each line provides the user id, the number of displays and the number of clicks associated to the current search session. Other information are available in the dataset but are not relevant for this study. The lines are not grouped by user and the same user can be found in different and separate search sessions. Due to the large number of simulations run in this section, we kept only the first 1 million users out of 22 million, sorted by lexicographic order on the user id. An extract of the dataset is shown in Table \ref{table:dataset-sample} and some statistics are available in Table \ref{table:dataset-stats}. Furthermore, the distribution of the number of clicks per user (knowing the user has clicked at least once) is displayed in Figure \ref{fig:nbclicks-per-user}. It illustrates the fact that this number of clicks cannot be approximated by a Bernoulli law.
\begin{table}[h]
\begin{center}
\begin{tabular}{||ccc||}
\hline \hline
UserId &	NbDisplays 	&	NbClicks \\
\hline \hline
10000244&1&0\\
10000148&3&1\\
10000089&1&0\\
1000026&6&0\\
1000002&1&0\\
1000002&1&0\\
10000315&1&0\\
10000925&3&2\\
10000185&1&0\\
\hline \hline
\end{tabular}
\caption{Dataset sample}
\label{table:dataset-sample}
\end{center}
\end{table}
\begin{figure}
\begin{center}
\includegraphics[width=\textwidth]{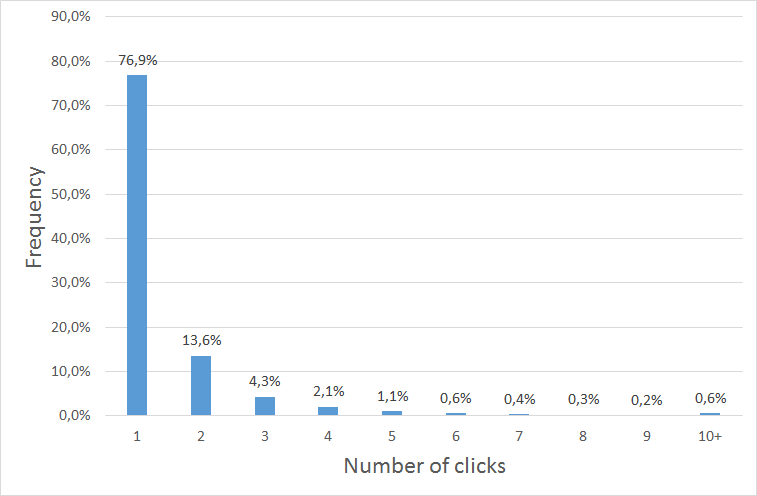}
\end{center}
\caption{Number of clicks per users with at least one click}
\label{fig:nbclicks-per-user}
\end{figure}

In the next Sections, blank AB-tests will be simulated from this dataset to compare the CTR (Click-Through Rate) of each population, defined as the average number of clicks per display:
$$\mbox{CTR} \eqdef \dfrac{\mbox{NbClicks}}{\mbox{NbDisplays}}\eqsp.$$
\begin{table}[h]
\begin{center}
\begin{tabular}{||ll||ll||}
\hline \hline
Number of users		&	$1000000$	&	CTR	&	$4.4\%$	\\
Number of displays	&	$4332627$	&	Displays per user	&	$4.3$	\\
Number of clicks		&	$191892$	&	Clicks per user	&	$0.19$	\\
\hline \hline
\end{tabular}
\caption{Dataset statistics}
\label{table:dataset-stats}
\end{center}
\end{table}

\subsection{User Independence Assumption}
\label{subsec:assumption-validation}
In order to validate Assumption A\ref{hypA:iid} and to show that it cannot be approximated by an independence of the displays, $500$ blank AB-tests were simulated\footnote{Experiments have also been made for $300$ and $400$ blank AB-tests and the results were very similar.}. For each AB-test, confidence intervals at different levels (from $50\%$ to $99\%$) were computed for the absolute CTR increment $\mbox{CTR}_{\testpop} - \mbox{CTR}_{\refpop}$ using two methods. The first one assumes that the \emph{displays} are independent implying an asymptotic variance of
$$\dfrac{\mbox{CTR}_{\refpop}(1-\mbox{CTR}_{\refpop})}{\mbox{NbDisplays}_{\refpop}} + \dfrac{\mbox{CTR}_{\testpop}(1-\mbox{CTR}_{\testpop})}{\mbox{NbDisplays}_{\testpop}} \eqsp.$$
This is the formula usually given when describing AB-test analysis. The second method assumes that the \emph{users} are independent and is described in Algorithm \ref{alg:clt-ci}. If the variables $(\Xa[i], \Ya[i], \Xb[i], \Yb[i])_{i\geq 1}$ model  the following quantities
\begin{itemize}
\item $\Xa[i]$: number of clicks from user $i$ if system $\refpop$ is applied,
\item $\Ya[i]$: number of displays shown to user $i$ if system $\refpop$ is applied,
\item $\Xb[i]$: number of clicks from user $i$ if system $\testpop$ is applied,
\item $\Yb[i]$: number of displays shown to user $i$ if system $\testpop$ is applied,
\end{itemize}
then the CTR of each population can be written
$$\mbox{CTR}_{\refpop} = \dfrac{\sumXtildea}{\sumYtildea} \eqsp, \quad \mbox{CTR}_{\testpop} = \dfrac{\sumXtildeb}{\sumYtildeb} \eqsp,$$
and the asymptotic variance of $\mbox{CTR}_{\testpop} - \mbox{CTR}_{\refpop}$ is given in Proposition \ref{prop:clt-ratio-absolute-diff}.

The true value of the absolute increment is known to be $0$ and, for each confidence level, we give the percentage of AB-tests for which the confidence interval contained $0$. The closer this percentage to the target confidence level, the better the underlying method. Results for both assumptions are shown in Figure \ref{fig:display-vs-user-iid}.
\begin{figure}
\begin{center}
\includegraphics[width=\textwidth]{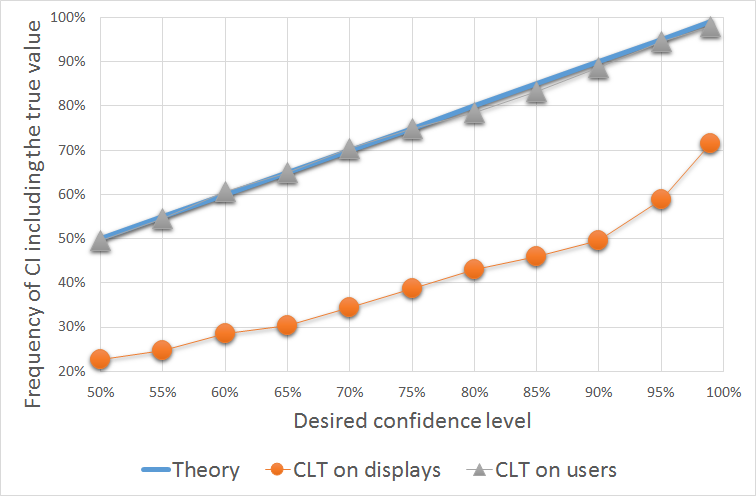}
\end{center}
\caption{Display VS User independence assumption}
\label{fig:display-vs-user-iid}
\end{figure}
Assuming independence of displays leads to under-estimating the AB-test noise, and increments appear significant much more often than they should be. For example, a $95\%$-confidence interval includes the true value in only $59\%$ of AB-tests which contradicts the definition of a confidence interval. On the contrary, the assumption of user independence leads to the expected conclusion of having almost $95\%$ of $95\%$-confidence intervals including $0$ and it remains true for all other tested levels.

This under-estimation is explicitly illustrated in Figure \ref{fig:nbclicks-distribution} where the empirical distribution of the number of clicks (obtained by bootstrapping) is compared to the binomial distribution implied by the display independence assumption.
\begin{figure}
\begin{center}
\includegraphics[width=\textwidth]{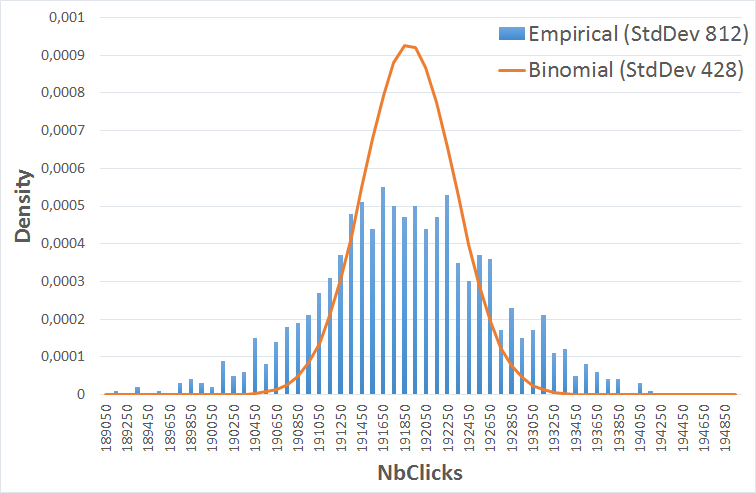}
\end{center}
\caption{Empirical vs binomial number of clicks distribution}
\label{fig:nbclicks-distribution}
\end{figure}
It shows that the empirical standard deviation is much higher than the binomial one (twice as big iN this example).

\subsection{Comparison of Bootstrap Algorithms}
\label{subsec:algo-comparison}
The assumption of independence by user having been validated, we can now focus on the comparison of the proposed algorithms. The method using only the central limit theorem will be given as a reference but is not of practical interest here as the dataset is not grouped by user (see Section \ref{subsec:bootstrap}). We are thus more interested in comparing Algorithms \ref{alg:ci-bootstrap} and \ref{alg:fast-ci} as they can be implemented in a such a way that the dataset is read only once. Each algorithm uses bootstrapping, having a computational cost linear in the number of bootstraps $M$. Similarly to Section \ref{subsec:assumption-validation}, $500$ blank AB-tests were simulated from the dataset described in \ref{subsec:dataset} to compute confidence intervals for CTR relative increment
$$\dfrac{\mbox{CTR}_{\testpop}}{\mbox{CTR}_{\refpop}}-1 = \dfrac{\sumXtildeb/\sumYtildeb}{\sumXtildea/\sumYtildea} - 1 \eqsp,$$
where $(\Xa[i], \Ya[i], \Xb[i], \Yb[i])_{i\geq 1}$ are defined in Section \ref{subsec:assumption-validation}. According to Proposition \ref{prop:clt-ratio-relative-diff}, this estimator is asymptotically normal and its average should be $0$ for a blank AB-test. The frequency of confidence intervals including the true value $0$ is displayed in Figure \ref{fig:bootstrap-levels} for different levels of confidence and for both the pure bootstrap technique with $M=10$ (Algorithm \ref{alg:ci-bootstrap}) and the technique using the bootstrap variance in the CLT (Algorithm \ref{alg:fast-ci}) again with $M=10$.
\begin{figure}
\begin{center}
\includegraphics[width=\textwidth]{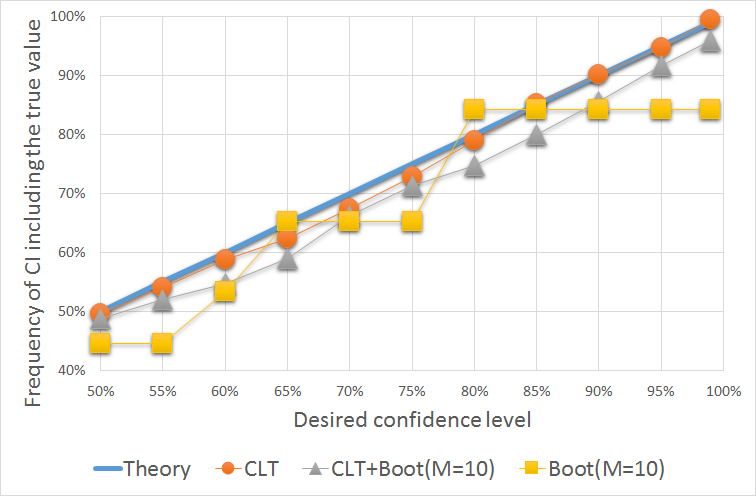}
\end{center}
\caption{Bootstrap algorithms' performance with $M=10$}
\label{fig:bootstrap-levels}
\end{figure}
As expected, for a small number of bootstraps $M=10$, the pure bootstrap algorithm performs poorly and is able to get an acceptable confidence intervals for only a few confidence levels, while the algorithm using both CLT and bootstrapping shows good results for all confidence levels for the same computational cost.
In Figure \ref{fig:bootstrap-number}, we show the influence of the number of bootstraps $M$ in the ability of each algorithm to compute reliable $95\%$ confidence intervals.
\begin{figure}
\begin{center}
\includegraphics[width=\textwidth]{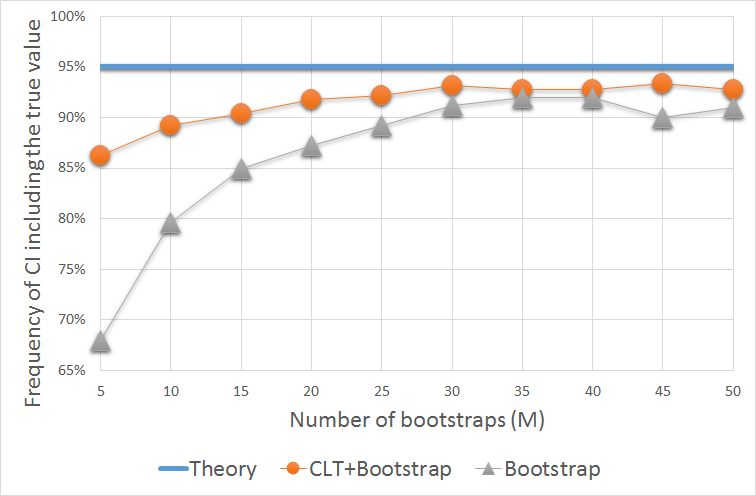}
\end{center}
\caption{Bootstrap algorithms' performance for different values of $M$}
\label{fig:bootstrap-number}
\end{figure}
The pure bootstrap algorithm converges more slowly to the target $95\%$ value and requires twice the computational cost as the mixed algorithm.

\section{Conclusion}
\label{sec:conclusion}
We have translated the AB-test process into a statistical framework, providing three algorithms for the computation of confidence intervals. Each of them are useful for different practical cases:
\begin{enumerate}
\item if the number of users $\numuid$ is small, pure bootstrapping is the best choice (see Algorithm \ref{alg:ci-bootstrap}), and a large number of bootstraps $M$ is tractable;
\item if the number of users $\numuid$ is large and the dataset is grouped by user, then one should use one of the relevant central limit theorems (see Algorithm \ref{alg:clt-ci});
\item if the number of users $\numuid$ is large and the dataset is not grouped by user, the algorithm using the bootstrap variance in the central limit theorem will result in the smallest computational cost (see Algorithm \ref{alg:fast-ci}).
\end{enumerate}
Numerical experiments allowed us to check that our assumptions were valid. We focused on the CTR computation, but, as stated in the theoretical parts, the proposed algorithms apply to any metric that can be written as a sum or a ratio of sums, e.g., to the sales amount spend per user as well as the revenue generated per user. Similar numerical results allowed us to validate the algorithms.

It is worthwhile to note that the provided algorithms lead to results valid only during the AB-test but do not extend to the future. This is known as the \emph{long term effect} as discussed in \citet{kohavi:longbotham:2009}. Addressing this issue would require additional assumptions on the metrics of interest, such as time series modeling, and is out of the scope of this paper.

\acks{The author would like to thank Olivier Chapelle, Alexandre Gilotte, Andrew Kwok and Nicolas Le Roux for their valuable ideas, comments and feedbacks.}

\newpage

\appendix
\section{Convergence Results}
\label{appendix:proofs}
The notations used here are independent from the ones defined in the other sections as the following propositions are general results on random variable convergence. We only keep the definition of $\nonzero$ given in Definition \ref{def:nonzero} that is widely used when dealing with ratios.

All the random variables will be assumed to be defined on a probability space $(\Omega, \mathcal{F}, \mathbb{P})$ and the expectation operator under $\mathbb{P}$ will be denoted by $\esp{\cdot}$.

\begin{lemma}\label{lem:nonzero-probaconv}
Let $(Y_n)_{n \geq 1}$ be a sequence a real random variables converging in probability to a real constant $y \neq 0$. Then the sequence $(\nonzero[Y_n])_{n\geq 1}$ also converges in probability to $y$ where $\nonzero$ is defined in Definition \ref{def:nonzero}.
\end{lemma}
\begin{proof}
By the triangle inequality, we have, for each $n \geq 1$
$$|\nonzero[Y_n] - y| \leq |\nonzero[Y_n] - Y_n| + |Y_n - y| = \one_{Y_n=0} + |Y_n - y|\eqsp,$$
implying that for each $\varepsilon > 0$
$$\proba{|\nonzero[Y_n] - y|>\varepsilon} \leq \proba{\one_{Y_n=0}>\varepsilon} + \proba{|Y_n - y|>\varepsilon}\eqsp,$$
where the second probability converges to $0$ by definition of $Y_n \plim y$ and the first one is bounded by
\begin{align*}
\proba{\one_{Y_n=0}>\varepsilon} & \leq \proba{Y_n=0} \eqsp, \\
	&	\leq \proba{|Y_n - y|>|y|/2}\eqsp,
\end{align*}
where the last probability converges to $0$ by definition of $Y_n \plim y$.
\end{proof}

\begin{lemma}\label{lem:nonzero-clt}
Let $(X_n^1, \cdots, X_n^d, Y_n)_{n \geq 1}$ be a sequence a random variables in  $\rset^{d+1}$ such that
\begin{enumerate}
\item \label{lem:ass:probaconv}  $Y_n \plim y$ where $y$ is real constant such that $y\neq 0$,
\item \label{lem:ass:nonzerospeed} There exists $c \in [0,1)$ such that $\proba{Y_n = 0} \leq c^n$,
\item \label{lem:ass:clt} There exist $(x_1, \cdots, x_d)\in\rset^d$ and a random variable $V$ in $\rset^{d+1}$
\begin{equation*}
\sqrt{n}(X_n^1-x_1, \cdots, X_n^d-x_d, Y_n-y) \dlim V \eqsp.
\end{equation*}
\end{enumerate}
then the assertions \ref{lem:ass:probaconv} and \ref{lem:ass:clt} are satisfied with $\nonzero[Y_n]$ where $\nonzero$ is defined in Definition \ref{def:nonzero}.
\end{lemma}
\begin{proof}
Assumption \ref{lem:ass:probaconv} and Lemma \ref{lem:nonzero-probaconv} directly give $\nonzero[Y_n] \plim y$.

In order to proove the distribution convergence, we use the portemanteau lemma by showing that for all bounded Lipschitz function $f$, $\esp{f(\sqrt{n}(X_n^1-x_1, \cdots, X_n^d-x_d, \nonzero[Y_n]-y))}$ converges to $\esp{f(V)}$. Let $f$ be a bounded and Lipschitz function, we have
\begin{multline}\label{eq:nonzero-clt:triangle}
\left|\esp{f(\sqrt{n}(X_n^1-x_1, \cdots, X_n^d-x_d, \nonzero[Y_n]-y))} - \esp{f(V)}\right| \\
\leq \esp{\left|f(\sqrt{n}(X_n^1-x_1, \cdots, X_n^d-x_d, \nonzero[Y_n]-y))-f(V)\right|} \eqsp, \\
\leq \esp{\left|f(\sqrt{n}(X_n^1-x_1, \cdots, X_n^d-x_d, \nonzero[Y_n]-y)-f(\sqrt{n}(X_n^1-x_1, \cdots, X_n^d-x_d, Y_n-y))\right|}
\\+ \esp{\left|f(\sqrt{n}(X_n^1-x_1, \cdots, X_n^d-x_d, Y_n-y))-f(V)\right|}\eqsp.
\end{multline}
According to Assumption \ref{lem:ass:clt}, the second term of the right hand side of \eqref{eq:nonzero-clt:triangle} converges to $0$. The first term is handled using the Lipschitz property of $f$: there exists a constant K such that for all $(a,b)\in(\rset^{d+1})^2$, $|f(a)-f(b)|_{L_1} \leq K ||a-b||_{L_1}$ so that
\begin{multline*}
\esp{\left|f(\sqrt{n}(X_n^1-x_1, \cdots, X_n^d-x_d, \nonzero[Y_n]-y)-f(\sqrt{n}(X_n^1-x_1, \cdots, X_n^d-x_d, Y_n-y))\right|}  \\
\leq K\sqrt{n} \esp{\left|\left|(X_n^1-x_1, \cdots, X_n^d-x_d, \nonzero[Y_n]-y)-(X_n^1-x_1, \cdots, X_n^d-x_d, Y_n-y)\right|\right|_{L_1}} \eqsp, \\
= K\sqrt{n} \esp{\left|\nonzero[Y_n]- Y_n\right|}
= K\sqrt{n} \esp{\one_{Y_n=0}}
= K\sqrt{n} \proba{Y_n=0}  \eqsp,\\
\leq K\sqrt{n} c^n \eqsp, \quad \mbox{according to Assumption \ref{lem:ass:nonzerospeed},}
\end{multline*}
which shows that the first term of the right hand side of \eqref{eq:nonzero-clt:triangle} converges to $0$ and that $\sqrt{n}(X_n^1-x_1, \cdots, X_n^d-x_d, \nonzero[Y_n]-y) \dlim V$.
\end{proof}

\begin{proposition}\label{prop:multiratio-clt}
Let $(X_n, Y_n, X'_n, Y'_n)_{n \geq 1}$ be a sequence a random variables in  $\rset^4$, $(x,y,x',y') \in \rset^4$ and $\Sigma$ a $4\times4$ covariance matrix such that
\begin{enumerate}
\item \label{prop:ass:nonzeroy} $y \neq 0$ and $y' \neq 0$,
\item \label{prop:ass:probaconv}  $Y_n \plim y$ and  $Y'_n \plim y'$,
\item \label{prop:ass:nonzerospeed} There exists $c \in [0,1)$ such that $\proba{Y_n = 0} \leq c^n$ and $\proba{Y'_n = 0} \leq c^n$,
\item \label{prop:ass:clt} The sequence $(X_n, Y_n, X'_n, Y'_n)_{n \geq 1}$ satisfies the following central limit theorem
\begin{equation*}
\sqrt{n}\left(\begin{array}{c}X_n-x\\ Y_n-y\\ X'_n-x'\\ Y'_n-y'\end{array}\right) \dlim \normallaw{\left(\begin{array}{c}0\\ 0\\ 0\\ 0\end{array}\right)}{\Sigma} \eqsp.
\end{equation*}
\end{enumerate}
Then the ratio sequence $(X_n/\nonzero[Y_n],X'_n/\nonzero[Y'_n])_{n \geq 1}$ satisfies the following central limit theorem
\begin{equation*}
\sqrt{n}\left(\begin{array}{c}\dfrac{X_n}{\nonzero[Y_n]}-\dfrac{x}{y}\\ \dfrac{X'_n}{\nonzero[Y'_n]}-\dfrac{x'}{y'}\end{array}\right) \dlim \normallaw{\left(\begin{array}{c}0\\ 0\end{array}\right)}{P^T\Sigma P} \eqsp,\quad \mbox{where}\quad
P \eqdef \left(\begin{array}{cc}
\dfrac{1}{y}	&	0	\\
-\dfrac{x}{y^2}	&	0	\\
0			&	\dfrac{1}{y'} 	\\
0			&	-\dfrac{x'}{(y')^2}
\end{array}\right)\eqsp.
\end{equation*}
\end{proposition}
\begin{proof}
We first rewrite $\frac{X_n}{\nonzero[Y_n]}-\frac{x}{y}$ as
\begin{align*}
\dfrac{X_n}{\nonzero[Y_n]}-\dfrac{x}{y} & = \dfrac{X_n}{\nonzero[Y_n]} - \dfrac{x}{\nonzero[Y_n]} +  \dfrac{x}{\nonzero[Y_n]} - \dfrac{x}{y} \eqsp, \\
	&	= \dfrac{y}{\nonzero[Y_n]} \dfrac{1}{y} (X_n-x) - \dfrac{y}{\nonzero[Y_n]} \dfrac{x}{y^2} (\nonzero[Y_n]-y) \eqsp,
\end{align*}
and simarly
$$\dfrac{X'_n}{\nonzero[Y'_n]}-\dfrac{x'}{y'} = \dfrac{y'}{\nonzero[Y'_n]} \dfrac{1}{y'} (X'_n-x') - \dfrac{y'}{\nonzero[Y'_n]} \dfrac{x'}{(y')^2} (\nonzero[Y'_n]-y') \eqsp,$$
so that
\begin{equation*}
\sqrt{n}\left(\begin{array}{c}\dfrac{X_n}{\nonzero[Y_n]}-\dfrac{x}{y}\\ \dfrac{X'_n}{\nonzero[Y'_n]}-\dfrac{x'}{y'}\end{array}\right) =
P_n^T
\sqrt{n}\left(\begin{array}{c}X_n-x\\ \nonzero[Y_n]-y\\ X'_n-x'\\ \nonzero[Y'_n]-y'\end{array}\right)\eqsp,
\end{equation*}
where
\begin{equation*}
P_n \eqdef
\left(\begin{array}{cc}
\dfrac{y}{\nonzero[Y_n]}\dfrac{1}{y}	&	0	\\
-\dfrac{y}{\nonzero[Y_n]}\dfrac{x}{y^2}	&	0	\\
0			&	\dfrac{y'}{\nonzero[Y'_n]}\dfrac{1}{y'} 	\\
0			&	-\dfrac{y'}{\nonzero[Y'_n]}\dfrac{x'}{(y')^2}
\end{array}\right)\eqsp.
\end{equation*}
By applying Lemma \ref{lem:nonzero-probaconv}, $\nonzero[Y_n] \plim y$ and  $\nonzero[Y'_n] \plim y'$ so that $P_n \plim P$. Furthermore, using Lemma \ref{lem:nonzero-clt} twice, we successively get that $(X_n, \nonzero[Y_n], X'_n, Y'_n)_{n \geq 1}$ and then $(X_n, \nonzero[Y_n], X'_n, \nonzero[Y'_n])_{n \geq 1}$ satisfy the CLT stated in Assumption \ref{prop:ass:clt}. We then only need to apply the Slutsky lemma to conclude.
\end{proof}

\begin{corollary}\label{cor:ratio-clt}
Let $(X_n, Y_n)_{n \geq 1}$ be a sequence a random variables in  $\rset^2$, $(x,y) \in \rset^2$ and $\Sigma$ a $2\times2$ covariance matrix such that
\begin{enumerate}
\item  $y \neq 0$,
\item  $Y_n \plim y$,
\item There exists $c \in [0,1)$ such that $\proba{Y_n = 0} \leq c^n$,
\item $(X_n, Y_n)_{n \geq 1}$ satisfies the following central limit theorem
\begin{equation*}
\sqrt{n}\left(\begin{array}{c}X_n-x\\ Y_n-y\end{array}\right) \dlim \normallaw{\left(\begin{array}{c}0\\ 0\end{array}\right)}{\Sigma} \eqsp.
\end{equation*}
\end{enumerate}
Then the ratio sequence $(X_n/\nonzero[Y_n])_{n \geq 1}$ satisfies the following central limit theorem
\begin{equation*}
\sqrt{n}\left(\dfrac{X_n}{\nonzero[Y_n]}-\dfrac{x}{y}\right) \dlim \normallaw{0}{P^T\Sigma P} \eqsp,\quad \mbox{where}\quad
P \eqdef \left(\begin{array}{c}
\dfrac{1}{y}\\
-\dfrac{x}{y^2}
\end{array}\right)\eqsp.
\end{equation*}
\end{corollary}
\begin{proof}
This is a direct consequence of Proposition \ref{prop:multiratio-clt} by keeping only the first marginal of the ratio couple.
\end{proof}

\section{Proofs of Central Limit Theorems}
\label{appendix:proofs-clt}

\begin{proof}[Proof of Theorem \ref{th:clt}]
The vector $(\sumXtildea, \sumYtildea, \sumXtildeb, \sumYtildeb)$ is made of empirical means of random variables $(\Xtildea[i], \Ytildea[i], \Xtildeb[i], \Ytildeb[i])_{i \geq 1}$.
According to Definition \ref{def:tilde} and to Assumption A\ref{hypA:iid}, these variables are i.i.d. and by the same definition they are centered on $(\mean{\Xa}, \mean{\Ya}, \mean{\Xb}, \mean{\Yb})$.
Furthermore, one directly sees that
\begin{equation*}
\abs{\Xtildea[1]} \leq \frac{1}{\ratioa} \abs{\Xa[1]} \eqsp, \eqsp \abs{\Ytildea[1]} \leq \frac{1}{\ratioa} \abs{\Ya[1]} \eqsp, \eqsp \abs{\Xtildeb[1]} \leq \frac{1}{\ratioa} \abs{\Xb[1]} \eqsp, \eqsp \abs{\Ytildeb[1]} \leq \frac{1}{\ratioa} \abs{\Yb[1]}\eqsp,
\end{equation*}
which, combined with Assumption A\ref{hypA:moment}, shows that $(\Xtildea[1], \Xtildeb[1], \Ytildea[1], \Ytildeb[1])$ is $L_2$-integrable. We then can apply a multi-dimensional version of the central limit theorem to get the announced convergence in distribution result.

It now only remains to calculate the related variances and covariances. By Definition\ref{def:tilde}, we have
\begin{align*}
\var{\Xtildea[1]} & = \var{\dfrac{\epsilona[1]\Xa[1]}{\ratioa}} \eqsp, \\
	& = \frac{1}{\ratioa^2} \left\{ \esp{(\epsilona[1])^2(\Xa[1])^2} - \esp{\epsilona[1]\Xa[1]}^2 \right\} \eqsp, \\
	& = \frac{1}{\ratioa^2}  \left\{ \esp{\epsilona[1]}\esp{(\Xa[1])^2} - \esp{\epsilona[1]}^2\esp{\Xa[1]}^2 \right\} \eqsp, \quad \mbox{by Assumption A\ref{hypA:abtest-independence},}\\
	& = \frac{1}{\ratioa}\esp{(\Xa[1])^2} - \mean{\Xa}^2 \eqsp, \\
	& = \frac{1}{\ratioa} \left[\stddev{\Xa}^2 + \mean{\Xa}^2\right] - \mean{\Xa}^2 \eqsp, \quad \mbox{according to assumation A\ref{hypA:moment},}\\
	& = \dfrac{1}{\ratioa}\stddev{\Xa}^2 + \dfrac{1-\ratioa}{\ratioa} \mean{\Xa}^2 \eqsp.
\end{align*}
The same stands for $\var{\Ytildea[1]}$, $\var{\Xtildeb[1]}$, and $\var{\Ytildeb[1]}$, and very similar steps allows to get the values of $\cov{\Xtildea[1]}{\Ytildea[1]}$ and $\cov{\Xtildeb[1]}{\Ytildeb[1]}$.

Using again Definition \ref{def:tilde} and Assumption A\ref{hypA:epsilon-law}, one gets
\begin{align*}
\cov{\Xtildea[1]}{\Xtildeb[1]} & = \esp{\dfrac{\epsilona[1]\Xa[1]}{\ratioa}\dfrac{\epsilonb[1]\Xb[1]}{\ratiob}} - \esp{\Xtildea[1]}\esp{\Xtildeb[1]} \eqsp, \\
	& = -\mean{\Xa}\mean{\Xb}\eqsp, \quad \mbox{as } \epsilona[1]\epsilonb[1] = 0 \eqsp,
\end{align*}
and the same formula can be derived for $\cov{\Xtildea[1]}{\Ytildeb[1]}$, $\cov{\Ytildea[1]}{\Xtildeb[1]}$ and $\cov{\Ytildea[1]}{\Ytildeb[1]}$.
\end{proof}

\begin{proof}[Proof of Proposition \ref{prop:clt-diff}]
Define a continuous function $g$ from $\rset^4$ to $\rset$ by
\begin{multline*}
\forall (x_{\refpop},y_{\refpop},x_{\testpop},y_{\testpop}) \in \rset^4\eqsp,\\
g\left((x_{\refpop},y_{\refpop},x_{\testpop},y_{\testpop})^T\right) \eqdef x_{\testpop} - x_{\refpop} = (-1,0,1,0) \times  (x_{\refpop},y_{\refpop},x_{\testpop},y_{\testpop})^T\eqsp,
\end{multline*}
so that
$$\sqrt{\numuid} \left[\left(\sumXtildeb- \sumXtildea\right) - \left(\mean{\Xb}-\mean{\Xa}\right)\right] = g\left(\sqrt{\numuid}
\left(
	\begin{array}{c}
		\sumXtildea - \mean{\Xa} \\
		\sumYtildea - \mean{\Ya} \\
		\sumXtildeb - \mean{\Xb} \\
		\sumYtildeb - \mean{\Yb}
	\end{array}
\right)
\right) \eqsp.$$
Then, by the continuous mapping theorem and Theorem \ref{th:clt}, $\sqrt{\numuid} \left[\left(\sumXtildeb- \sumXtildea\right) - \left(\mean{\Xb}-\mean{\Xa}\right)\right]$ converges in distribution to a normal random variable of mean $0$ and variance
$$(-1,0,1,0)\Sigma\left(\Xtildea[1], \Ytildea[1], \Xtildeb[1], \Ytildeb[1] \right)(-1,0,1,0)^T\eqsp.$$
\end{proof}

Before moving to proofs of ratio CLT, we need two intermediary Lemmas.
\begin{lemma} \label{lem:positive-means}
Under Assumptions A\ref{hypA:moment}-\ref{hypA:non-negative}, we have
\begin{equation*}
\mean{\Xa} > 0 \eqsp, \quad \mean{\Ya} > 0 \eqsp, \quad \mean{\Xb} > 0 \eqsp, \quad \mean{\Yb} > 0 \eqsp.
\end{equation*}
\end{lemma}
\begin{proof}
According to assumpation A\ref{hypA:non-negative}, $\Xa[1] \geq 0$ almost surely, which implies that $\mean{\Xa} = \esp{\Xa[1]} \geq 0$. Furthermore, by the Markov inequality, for any $n \geq 1$ we have:
$$\proba{\Xa[1] \geq 1/n} \leq n \mean{\Xa}\eqsp.$$
If $\mean{\Xa} = 0$ then for any $n \geq 1$, $\proba{\Xa[1] \geq 1/n} = 0$ and thus $\proba{\Xa[1] > 0} = 0$ which is in contradiction with Assumption A\ref{hypA:non-negative}.
\end{proof}

\begin{lemma}\label{lem:proba-sum-0}
Under Assumptions A\ref{hypA:iid}-\ref{hypA:epsilon-law} there exists a constant $c \in [0,1)$ such that
$$\proba{\sumXtildea = 0} \leq c^{\numuid}\eqsp,\quad \proba{\sumYtildea = 0} \leq c^{\numuid}\eqsp,\quad \proba{\sumXtildeb = 0} \leq c^{\numuid}\eqsp,\quad \proba{\sumYtildeb = 0} \leq c^{\numuid}\eqsp.$$
\end{lemma}
\begin{proof}
We have
\begin{align*}
\proba{\sumXtildea = 0} & = \proba{\Xtildea[1] = 0}^{\numuid}\eqsp,	\\
	&	 = \proba{\epsilona[1]\Xa[1] = 0}^{\numuid}\eqsp,	\quad \mbox{by Definition \ref{def:tilde},}\\
	&	 = \left[1-\proba{\epsilona[1]\Xa[1] > 0}\right]^{\numuid}\eqsp, \\
	&	 = \left[1-\proba{\epsilona[1] > 0, \Xa[1] > 0}\right]^{\numuid}\eqsp, \\
	&	 = \left[1-\proba{\epsilona[1] > 0}\proba{\Xa[1] > 0}\right]^{\numuid}\eqsp,  \quad \mbox{by Assumption A\ref{hypA:abtest-independence},}\\
	&	 = \left[1-\ratioa\proba{\Xa[1] > 0}\right]^{\numuid}\eqsp,  \quad \mbox{by Assumption A\ref{hypA:epsilon-law},}
\end{align*}
where $1-\ratioa\proba{\Xa[1] > 0} \in [0, 1)$ by Assumption A\ref{hypA:non-negative}. The same steps applied to $\sumYtildeb$, $\sumXtildea$ and $\sumYtildeb$ achieve the proof by setting
\begin{equation*}
c \eqdef 1 - \mbox{min}\left[\ratioa\proba{\Xa[1] > 0}, \ratioa\proba{\Ya[1] > 0}, \ratiob\proba{\Xb[1] > 0}, \ratiob\proba{\Yb[1] > 0}\right] \eqsp.
\end{equation*}
\end{proof}

\begin{proof}[Proof of Proposition \ref{prop:clt-relative-diff}]
The proof is a direct application of Corollary \ref{cor:ratio-clt} of Appendix \ref{appendix:proofs} with $X_n = \sumXtildeb$ and $Y_n = \sumXtildea$. Its assumptions are all satisfied:
\begin{enumerate}
\item $\mean{\Xa} \neq 0$ by Lemma \ref{lem:positive-means},
\item According to the weak law of large numbers, $\sumXtildea \plim \mean{\Xa}$,
\item $\proba{\sumXtildea = 0} \leq c^{\numuid}$ according to Lemma \ref{lem:proba-sum-0}
\item According to Theorem \ref{th:clt}
\begin{equation*}
\sqrt{\numuid}\left(\begin{array}{c}
\sumXtildeb - \mean{\Xb} \\
\sumXtildea - \mean{\Xa}
\end{array}\right)
\dlim
\normallaw{
\left(\begin{array}{c}
0\\
0
\end{array}\right)
}{
\left(\begin{array}{cc}
\stddevXtildeb^2	&	-\mean{\Xa}\mean{Xb}	\\
-\mean{\Xa}\mean{Xb}&	\stddevXtildea^2
\end{array}\right)
}\eqsp.
\end{equation*}
\end{enumerate}
Corollary \ref{cor:ratio-clt} states that
\begin{equation*}
\sqrt{\numuid}\left( \dfrac{\sumXtildeb}{\nonzero[\sumXtildea]} - \dfrac{\mean{\Xb}}{\mean{\Xa}}\right) \dlim \normallaw{0}{P^T\left(\begin{array}{cc}
\stddevXtildeb^2	&	-\mean{\Xa}\mean{Xb}	\\
-\mean{\Xa}\mean{Xb}&	\stddevXtildea^2
\end{array}\right)P
}\eqsp,
\end{equation*}
where $P \eqdef \left(\dfrac{1}{\mean{\Xa}}, -\dfrac{\mean{\Xb}}{\mean{\Xa}^2}\right)^T$.
\end{proof}

\begin{proof}[Proof of Proposition \ref{prop:clt-ratio}]
The proof is a direct application of Proposition \ref{prop:multiratio-clt} of Appendix \ref{appendix:proofs} with $X_n = \sumXtildea$, $Y_n = \sumYtildea$,  $X'_n = \sumXtildeb$, and $Y'_n = \sumYtildeb$. Its assumptions are all satisfied:
\begin{enumerate}
\item $\mean{\Ya} \neq 0$ and $\mean{\Yb} \neq 0$ by Lemma \ref{lem:positive-means},
\item According to the weak law of large numbers, $\sumYtildea \plim \mean{\Ya}$ and $\sumYtildeb \plim \mean{\Yb}$,
\item $\proba{\sumYtildea} \leq c^{\numuid}$ and  $\proba{\sumYtildeb} \leq c^{\numuid}$ by Lemma \ref{lem:proba-sum-0},
\item According to Theorem \ref{th:clt}, the CLT condition is satisfied.
\end{enumerate}
Proposition \ref{prop:multiratio-clt} states that
\begin{equation*}
\sqrt{\numuid}\left(\begin{array}{c}
\dfrac{\sumXtildea}{\nonzero[\sumYtildea]} - \dfrac{\mean{\Xa}}{\mean{\Ya}} \\
\dfrac{\sumXtildeb}{\nonzero[\sumYtildeb]} - \dfrac{\mean{\Xb}}{\mean{\Yb}}
\end{array}\right)
 \dlim \normallaw{0}{P^T\Sigma\left(\Xtildea[1], \Ytildea[1], \Xtildeb[1], \Ytildeb[1] \right)P}\eqsp,
\end{equation*}
where $\Sigma\left(\Xtildea[1], \Ytildea[1], \Xtildeb[1], \Ytildeb[1] \right)$ is defined in Theorem \ref{th:clt} and
\begin{equation*}
P \eqdef \left(
\begin{array}{cc}
\dfrac{1}{\mean{\Ya}}	&	0	\\
-\dfrac{\mean{\Xa}}{\mean{\Ya}^2}	&	0	\\
0	&	\dfrac{1}{\mean{\Yb}}	\\
0	&	-\dfrac{\mean{\Xb}}{\mean{\Yb}^2}
\end{array}
\right)\eqsp.
\end{equation*}
\end{proof}

\begin{proof}[Proof of Proposition \ref{prop:clt-ratio-absolute-diff}]
The proof follows the same steps as the one of Proposition \ref{prop:clt-diff}.
\end{proof}

\begin{proof}[Proof of Proposition \ref{prop:clt-ratio-relative-diff}]
The proof is another application of Corrolary \ref{cor:ratio-clt} in Appendix \ref{appendix:proofs} with $X_n = \sumXtildeb/\nonzero[\sumYtildeb]$ and $Y_n =\sumXtildea/\nonzero[\sumYtildea]$ for which we check the assumptions:
\begin{enumerate}
\item $\mean{\Xa} / \mean{\Ya} \neq 0$ by Lemma \ref{lem:positive-means},
\item By the weak law of large numbers, we have $\sumXtildea \plim \mean{\Xa}$ and $\sumYtildea \plim \mean{\Ya}$. Then by Lemma \ref{lem:nonzero-probaconv}, $\nonzero[\sumYtildea] \plim \mean{\Ya}$ and we can apply the continuous mapping theorem to get $\sumXtildea/\nonzero[\sumYtildea] \plim \mean{\Xa}/\mean{\Ya}$,
\item According to Lemma \ref{lem:proba-sum-0}, we have $\proba{\sumXtildea/\nonzero[\sumYtildea] = 0} = \proba{\sumXtildea = 0} \leq c^{\numuid}$,
\item The central limit theorem is stated in Proposition \ref{prop:clt-ratio}.
\end{enumerate}
\end{proof}

\vskip 0.2in

\end{document}